%% file: 2021RSSSuBaTrMu.tex
\documentclass[conference]{IEEEtran}
\usepackage{times}

\usepackage[numbers]{natbib}
\usepackage{multicol}
\usepackage[bookmarks=true]{hyperref}

\usepackage{amsmath}
\usepackage{graphicx}
\usepackage{mathtools}
\usepackage{amssymb}
\usepackage{hyperref}
\hypersetup{
	colorlinks=true,
	linkcolor=blue,
	filecolor=magenta,      
	urlcolor=cyan,
}
\usepackage[numbers]{natbib}
\usepackage{multicol}
\usepackage[bookmarks=true]{hyperref}

\usepackage{wrapfig}
\usepackage{tikz}                                                          
\usepackage{graphicx} 
\usepackage{amssymb}
\usepackage{amsthm}
\usepackage{amsmath}
\usepackage{epstopdf}
\usepackage{graphics}
\usepackage{amsmath, amssymb}
\numberwithin{equation}{section}
\usepackage{color}      
\usepackage{subfig}  
\usepackage{hyperref}   
\usepackage{graphicx}   
\usepackage{verbatim}   
\usepackage{xspace}
\usepackage{float}
\usepackage[font={footnotesize}]{caption}
\usepackage{csquotes}
\usepackage{footnote}
\makesavenoteenv{tabular}
\makesavenoteenv{table}
\usepackage[]{threeparttable}
\usepackage{breqn}

\usepackage{amsmath}

\DeclareMathOperator*{\argmax}{arg\,max}
\DeclareMathOperator*{\argmin}{arg\,min}

\newcommand{\F}{\mathcal{F}}
\newcommand{\f}{\mathbf{f}}
\newcommand{\mL}{\mathcal{L}}
\newcommand{\R}[1]{\mathbb{R}^{#1}}

\newcommand{\I}{\mathcal{I}}

\newcommand{\X}{\mathcal{X}}

\renewcommand{\I}{\mathcal{I}}

\newcommand{\N}{\mathcal{N}}

\usepackage{amsthm}
\newtheorem{theorem}{Theorem}[section]
\newtheorem{lemma}{Lemma}[section]
\newtheorem{definition}{Definition}

\usepackage{listings}
\usepackage{color}

\definecolor{dkgreen}{rgb}{0,0.6,0}
\definecolor{gray}{rgb}{0.5,0.5,0.5}
\definecolor{mauve}{rgb}{0.58,0,0.82}

\definecolor{codegreen}{rgb}{0,0.6,0}
\definecolor{codegray}{rgb}{0.5,0.5,0.5}
\definecolor{codepurple}{rgb}{0.58,0,0.82}
\definecolor{backcolour}{rgb}{0.95,0.95,0.92}

\usepackage[ruled, lined, linesnumbered, commentsnumbered, longend]{algorithm2e}

\SetCommentSty{mycommfont}

\usepackage{ifthen,version}

\allowdisplaybreaks
  
\begin{document}

\title{Move Beyond Trajectories: \\
Distribution Space Coupling for Crowd Navigation}

\author{\authorblockN{Muchen Sun\authorrefmark{1},
Francesca Baldini\authorrefmark{2}\authorrefmark{3}
Peter Trautman\authorrefmark{2},
and Todd Murphey\authorrefmark{1}}
\authorblockA{\authorrefmark{1}Department of Mechanical Engineering, Northwestern University, Evanston, IL 60208, USA}
\authorblockA{\authorrefmark{2}Honda Research Institute, San Jose, CA 95134, USA}
\authorblockA{\authorrefmark{3}California Institute of Technology, Pasadena, CA 91125, USA}}

\maketitle

\begin{abstract}

Cooperatively avoiding collision is a critical functionality for robots navigating in dense human crowds, failure of which could lead to either overaggressive or overcautious behavior. A necessary condition for cooperative collision avoidance is to couple the prediction of the agents' trajectories with the planning of the robot's trajectory.  However, it is unclear that \emph{trajectory} based cooperative collision avoidance captures the correct agent attributes. In this work we migrate from trajectory based coupling to a formalism that couples agent preference \emph{distributions}. In particular, we show that preference distributions (probability density functions representing agents' intentions) can capture higher order statistics of agent behaviors, such as willingness to cooperate. Thus, coupling in distribution space exploits more information about inter-agent cooperation than coupling in trajectory space. We thus introduce a general objective for coupled prediction and planning in distribution space, and propose an iterative best response optimization method based on variational analysis with guaranteed sufficient decrease. Based on this analysis, we develop a sampling-based motion planning framework called \emph{DistNav}\footnote{For more details please visit \url{https://sites.google.com/view/distnav/}} that runs in real time on a laptop CPU. We evaluate our approach on challenging scenarios from both real world datasets and simulation environments, and benchmark against a wide variety of model based and machine learning based approaches. The safety and efficiency statistics of our approach outperform all other models.  Finally, we find that DistNav is competitive with \emph{human} safety and efficiency performance.

\end{abstract}

\IEEEpeerreviewmaketitle

\input{input/introduction}

\input{input/related-work}

\input{input/problem-formulation}

\input{input/distribution-optimization}

\input{input/evaluation}

\input{input/conclusion}

\section*{Acknowledgments}

This material is supported by the NSF Grant CNS 1837515. Any opinions, findings and conclusions or recommendations expressed in this material are those of the authors and
do not necessarily reflect the views of the aforementioned institutions.


\bibliographystyle{plainnat}
\bibliography{bibliography}

\input{input/appendix}

\end{document}

%% file: input/introduction.tex
\section{Introduction} \label{sec: introduction}
\begin{figure}
    \centering
    \includegraphics[width=200pt]{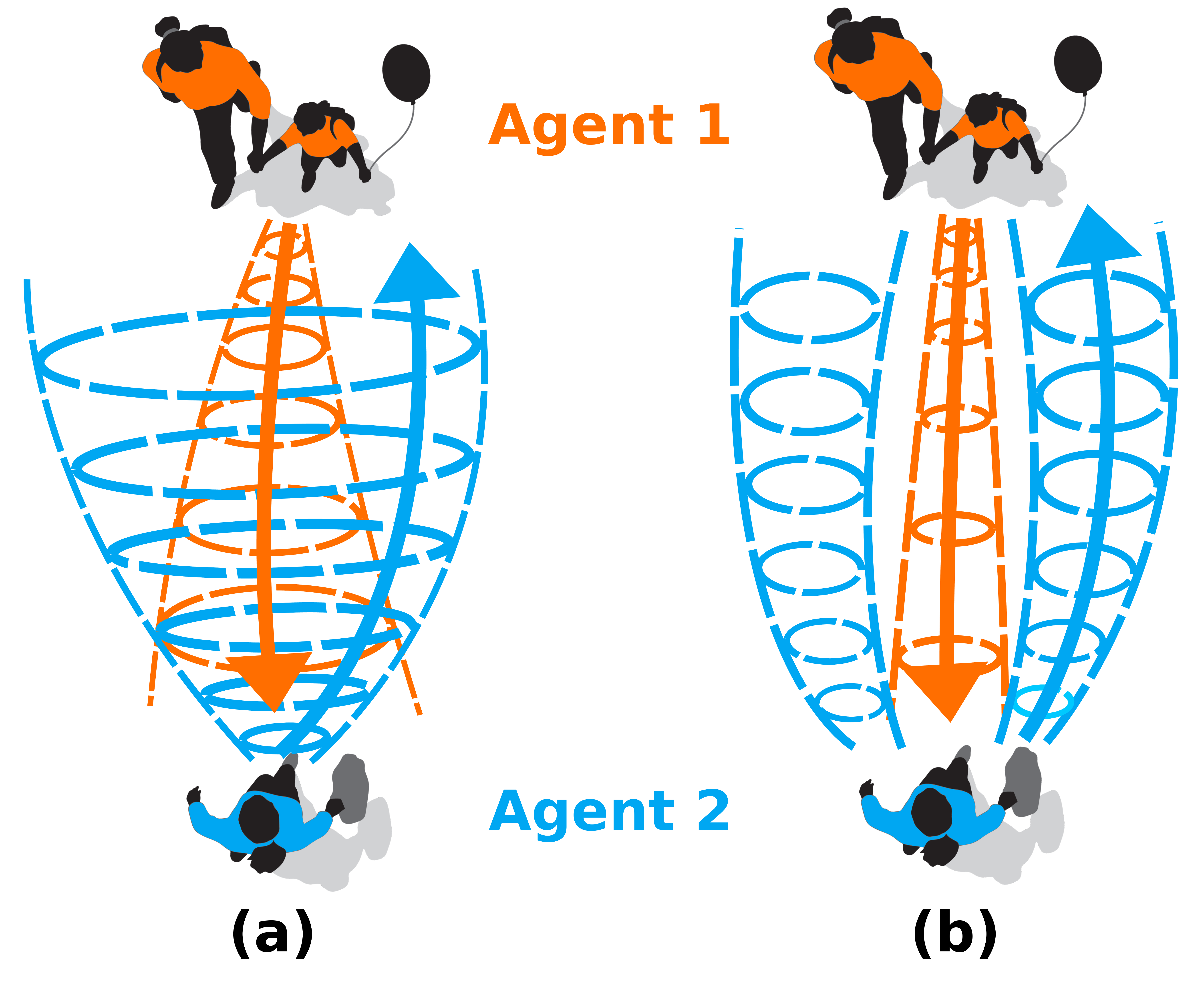}
    \caption{Difference between modeling humans as specific trajectories and distributions: Agent 2 at the bottom (pictured as a human but also can be a robot) needs to plan a path while predicting how the other two pedestrians, who are a family and thus considered as a single agent 1, would react to agent 2's decision. The dashed circles represent each agent's probability distribution (preference) at each time step, here we assume agent 1 has a lower flexibility and thus agent 1 would expect agent 2 to cooperate more to make space for each other. \textbf{(a) By modeling humans as \emph{trajectories}, agent 2 implicitly assumes the preferences of both agents are fixed in the presence of interaction, which is not necessarily true. (b) Modeling humans as \emph{distributions} can overcome this issue, where agent 2's intention evolves to a bi-modal distribution because of the potential interaction, and the trajectory with maximum likelihood (go right) would eventually be chosen as the plan.} While both methods can recover a proper path for agent 2 to avoid splitting up the family, \textbf{explicitly modeling the evolution of preferences captures more information about the interaction.}}
    \label{fig:trajectory_vs_distribution}
    \vspace{-15pt}
\end{figure}

\begin{figure*}
    \centering
	\includegraphics[width=\textwidth]{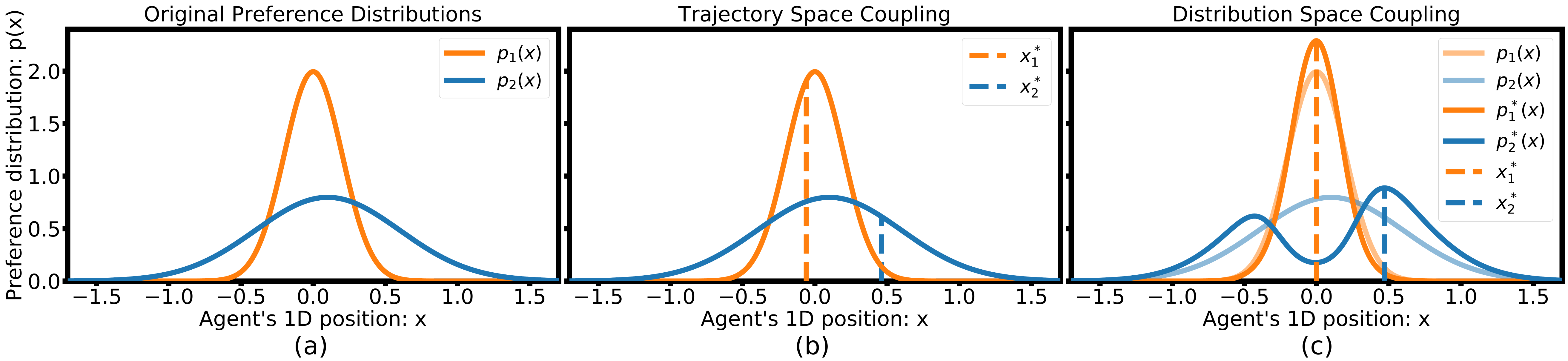}
    \caption{Simplified 1D example showing two agents bypassing each other, in this example each agent only considers its 1D location at next one step. (a) Original preference distributions of the two agents, modeled as two Gaussians. (b) Coupled prediction and planning in trajectory space, where agent 2 finds its optimal plan (blue dashed line on the right) and optimal prediction for agent 1's plan (orange dashed line near the center) simultaneously. Again, agent 2 implicitly assumes two agents' preferences remain static in presence of interaction. (c) \textbf{Our method relaxes this assumption, and allows preference distributions to evolve in a non-parametric manner}. Agent 2's intention can be explicitly modeled as a bi-modal distribution, it can both go right or go left, but going right is more preferred. Optimal plan (blue dashed line) and prediction (orange dashed line) are then picked by agent 2 from its belief for updated preferences.}
	\label{fig: conflict_1}
	\vspace{-15pt}
\end{figure*}  

Collision avoidance in dense human crowds is a challenging problem. Whereas conventional motion planning algorithms work well with slowly moving obstacles and low obstacle density, they are designed to work with passive obstacles that will not react to the robot, and assume the robot has full knowledge about the obstacles' future states. Such assumptions, however, no longer hold in human crowds and could lead to either overcautious or overaggressive robot behaviors, which is known as the freezing robot problem \cite{trautman-ijrr-2015}. Previous work on crowd navigation reveals that one necessary condition to avoid freezing robot behavior is to couple the prediction of the agent's future trajectories with robot planning (e.g., accounting for the robot's influence on agent behavior).  

However, pure trajectory-based models make strong implicit assumptions about the distribution governing interactive agent behaviors (what we call the agent's ``preference distribution'').  For example, the \emph{extent} to which an agent is willing to move out of the way of the robot is typically left unmodeled.   Further, trajectory space approaches assume that the (implicitly modeled)  preference distribution is a static quantity: \emph{during} interaction, the preference distribution does not change.  As an example, consider unimodal Gaussian agent models, where the covariance can be interpreted as the extent to which an agent will make room for another agent.  Critically, as two agents proceed past each other, their covariance will change. If an agent moves far out of the way of the path of the robot, its covariance will become more peaked, indicating that the agent is only willing to make so much room. For non-Gaussian agent distributions, higher order factors such as skew (e.g., preference for left versus right passage) and multi-modality (e.g., ambiguity about which side an agent might pass on), are themselves coupled to the deformation of the agent trajectories. In short, assuming that interaction occurs \emph{only} at the trajectory level ignores critical information about how interaction deforms higher order factors of the agent model (e.g., how interaction deforms the preference distribution). Figure~\ref{fig:trajectory_vs_distribution} illustrates the necessity of modeling agent preference evolution through a high-level example that is common in crowd navigation---two agents passing each other on the street.

We relax this assumption of a known and fixed preference distribution.  Specifically, we do not require the agent's preference distribution to remain static during interaction, but instead allow the preference distribution to evolve (Figure~\ref{fig: conflict_1} demonstrates the preference evolution in a 1D example). We accomplish this by:
\begin{enumerate}
    \item Taking the atomic representational unit to be a probability density function over agent trajectories, in what we call \emph{distribution space}.
    \item Including both human-robot and human-human interaction.
    \item Developing a method to jointly optimize over both robot and agent preference distributions. 
\end{enumerate}

In this work distributions represent agents' preferences over trajectories---including both what they want and their potentially multi-modal degree of flexibility.  This use of distributions is distinct from that found in other computational methods that use distributions are their atomic unit, such as belief space planning, where the purpose of a distribution is to represent uncertainty (for example,~\cite{rssac-crowd-nav} uses belief space planning for crowd navigation but decouples prediction and planning).  Here we develop a formalism and algorithm to ``shape'' probability density functions for better cooperation, even for systems with no uncertainty present.  

At last, we summarize our contributions as follow:
\begin{enumerate}
    \item We formulate a general objective for coupled prediction and planning in distribution space (Section \ref{sec: problem-formulation}). 
    \item We propose an iterative best response optimization algorithm which is guaranteed to decrease the objective in each iteration, and design a sampling-based navigation framework called DistNav (Section \ref{sec: methodology}).
    \item We conduct a comprehensive evaluation using both real world datasets and simulation environments, and benchmark DistNav against a variety of other crowd navigation methods. Results show that DistNav outperforms other methods in both safety and efficiency metrics, and is competitive with \emph{human} safety and efficiency performance (Section \ref{sec: evaluation}).
\end{enumerate}

%% file: input/related-work.tex
\section{Related Work} \label{sec: related-work}
Roboticists have been investigating navigation in human environments since the 1990s.  Two landmark studies were the RHINO~\cite{rhino-robot} and MINERVA~\cite{minervapaper} experiments. In~\cite{dutoit-tro} the authors observed that assuming agent independence leads to an uncertainty explosion that makes efficient navigation impossible. In~\cite{trautman-ijrr-2015}, it was shown that bounding uncertainty (such as in~\cite{navintent,dutoit-tro,joseph-dps-over-gps,gpmlras}) cannot prevent freezing robot behavior. Relatedly, \emph{Human intention aware} planning is a popular crowd navigation approach~\cite{crowd-nav-survey}; examples include \cite{hall-1966,mead-proxemic,svenstrup-human,risk-rrt,vaibhav-co-nav}. Although these approaches model human-robot interaction, they ignore human-robot \emph{cooperation}.

Some approaches learn navigation strategies by observing examples, such as through inverse reinforcement learning~\cite{kuderer-ijrr-2016, ziebartcabbie, ziebartppp}  or deep reinforcement learning~\cite{crowd-nav-deep-learning-mit,unfrozen-unlost,social-nav-gan-end-to-end}. Typically, human relationship models are ignored; importantly,~\cite{alahi-crowd-attention} models human-human interaction in a method called ``socially aware reinforcement learning (SARL).''  In~\cite{rgl-mprgl}, agent relational graphs are trained using deep reinforcement learning. In~\cite{social-nce}, a relational graph is paired with ``negative examples'' (e.g., collisions) to enforce safety constraints and social norms. These reinforcement learning methods rely on simulators for training, but as being pointed out in \cite{fraichard2020crowd}, current simulators make unrealistic but critical assumptions that will not hold in real world, which may affect the sim-to-real transfer.  

Coupled prediction and planning approaches explicitly capture the mutual dependencies between human and robot. An important body of work is game theoretic planning\cite{ilqgames, cleac2019algames}, these planners typically assume agent objectives are known, but recent works introduce online estimation of human agents' objectives \cite{schwarting2019social, lecleac2021lucidgames}. Further, \cite{schwarting2019social} uses the social value orientation (SVO) to quantify the degree of agents' selfishness or altruism. Similar idea of modeling human willingness to coordinate can be found in other works. For example, \cite{bam-journal} models mutual adaptation between human and robot in manipulation tasks, and \cite{trautman-icaps-2020} models human flexibility as the covariance matrix of a Gaussian process. In addition to online estimation, \cite{dorsa-auro-2018} introduces an active information gathering algorithm that generates communicative behaviors for autonomous vehicles. However, in these works the planners still over-confidently predict human behavior as a single trajectory, and none of them capture how willingness to coordinate changes \emph{during} interaction. 

Lastly, while the deep learning approaches in~\cite{lstm-crowd-prediction-eth,social-lstm,social-gan,social-attention,drl-pan} focus on \emph{prediction}, they are an important contribution to crowd navigation. In~\cite{pavone-multi-modal-generative}, variational auto encoders capture multimodality. \cite{chai2019multipath} includes human ``intent uncertainty'' and ``control uncertainty'' as part of the prediction, and model them with Gaussian mixture models. But none of these models explicitly measure or account for how flexibility changes during interaction.

%% file: input/problem-formulation.tex
\section{Problem Formulation} \label{sec: problem-formulation}

\subsection{Terminology}

Consider there are $n+1$ agents including $n$ pedestrians and one robot in the environment, we start by defining a set of unique indices $\I=\{R, 1, 2, \dots, n\}$ for all agents, where $R$ is the index of the robot and is treated as \emph{zero} when compared with other indices. The state of each agent is in the space $\X\subseteq\R{k}$, for example if we only consider all agents' planar positions, then $\X\subseteq\R{2}$. The trajectory of each agent $f^{(i)} : \R{+}\mapsto\X, i\in\I$ is a set function that maps a set of $T$ time points to the agent state at that moment. In practice, the trajectory $f$ would be evaluated as a vector or 2D matrix, thus dimension of the \emph{trajectory space} is $f\in\mathcal{F} \subseteq \R{k\times T}$. Agent states are measured through a measurement model $z^{(i)}_{t} = h(f^{(i)}(t))$, and we collect measurements of the agents' past states at time steps $\{1, 2, \dots, t\}$ as $z^{(i)}_{1:t} = [z^{(i)}_{1}, z^{(i)}_{2}, \dots, z^{(i)}_{t}]$, note that the observations could be noisy (in this paper we assume additive zero-mean Gaussian noises) and we don't assume $z^i_{1:t}$ to be complete: $z^{i}_{\tau}$ could be missing for some $\tau\in[1,t]$ and $i\in\I$. 

\begin{definition}[Distribution space]
The distribution space $\mathcal{P}$ is a function space, where each element $p(f):\mathcal{F}\mapsto\R{+}_0$ is a probability density function that maps the trajectory space $\mathcal{F}$ to the non-negative real domain, and each element satisfies:
\begin{align}
    \int_{\F} p(f) df = 1 \label{eq: distribution_constraint}
\end{align}
\end{definition}

\begin{definition}[Original preference distribution]
The prior probability for agent $i$'s trajectory conditioned on its measurements $z^{(i)}_{1:t}$ is defined as the agent's original preference distribution $p_i(f)=p(f^{(i)}=f|z^{(i)}_{1:t})\in\mathcal{P}$.
Since no measurement of other agents is used, original preference distribution doesn't include the interaction with other agents. 
\end{definition}

In this paper we use Gaussian processes regression to compute the original preferences, so in the rest of the paper we assume the original preferences are GPs, but our results and algorithm apply to \emph{arbitrary distributions}. We refer the readers to \cite{gpmlras} for more details about GP regression.

Preference distribution represents the agent's intention, which contains information for agent's preferred trajectories (intents) and their willingness to give up the preferred trajectories in order to cooperate with other agents (flexibility). Below we define \emph{intent} and \emph{flexibility} for Gaussian and arbitrary preferences.

\begin{definition}[Intent and flexibility for Gaussian preferences] \label{def: gp_shorthand}
When the agent's preference is a Gaussian process (GP) $p_i(f)=\N(f|\mu_i,\Sigma_i)$, the agent's intent and flexibility are defined as the GP mean $\mu_i$ and covariance $\Sigma_i$, respectively.
\end{definition}

\begin{definition}[Intent and flexibility for arbitrary distributions]
More generally, for any distribution $p(f)$, the intents are defined as the local maximums of $p(f)$ (so there could be multiple intents) and the flexibility is qualitatively measured through the covariance, skew and kurtosis of the distribution.
\end{definition}

Measuring flexibility for arbitrary distribution is tricky, in this paper we only consider qualitative analysis: Covariance represents the ``spread'' of the distribution, a larger covariance indicates a larger feasible action region. Skew measures the symmetry of the preference, for example whether the agent is more willing to go ``left'' or go ``right''. Kurtosis is the ``tailedness'' of the distribution, it can be interpreted as the agent's tolerance for large deviations from the intents. 

At last, we also need to define a function to penalize the likelihood of collision between two trajectories.

\begin{definition}[Collision penalty function] \label{def: collision_checking}
The collision penalty function $\psi(f^{(i)}, f^{(j)}):\F \mapsto \R{+}$ represents the likelihood of collision between two trajectories, and it needs to be symmetric such that $\psi(f^{(i)}, f^{(j)}) = \psi(f^{(j)}, f^{(i)})$.
\end{definition}
As an example, in our implementation we choose the collision penalty function to be:
\begin{align}
    \psi(f^{(i)}, f^{(j)}) = \max_{t} w \cdot \mathcal{N}(f^{(i)}(t) | f^{(j)}(t), \Sigma_{\psi})
\end{align} where $w\in\R{}$ is the penalty weight and $\Sigma_{\psi}\in\R{k\times k}$ controls how close two agents can be.

\subsection{Coupled Prediction and Planning in Distribution Space}

Coupled prediction and planning (also called generative navigation) considers motion planning as a prediction problem; when applied to crowd navigation, the robot couples the prediction of pedestrian trajectories and planning for its own trajectory by optimizing the joint trajectories of all agents simultaneously. The navigation goal and waypoints for the robot can be included as (artificial) observations, so the predicted robot trajectory will pass through them. We refer the readers to~\cite{trautman-ijrr-2015} for more details about incorporating navigation task into prediction. 

We start formulating coupled prediction and planning in distribution space by extending the collision penalty function to distribution space.
\begin{definition}[Expected collision penalty]
The expected collision penalty $c(p_i, p_j):\mathcal{P}\mapsto\R{+}_0$ is defined as the joint expected value of the collision penalty function $\psi(f^{(i)}, f^{(j)})$ with respect to two preference distributions:
\begin{align}
    c(p_i, p_j) = \int_{\F}\int_{\F} \psi(f^{(i)}, f^{(j)}) p_i(f^{(i)}) p_j(f^{(j)}) df^{(i)} df^{(j)} \label{eq: expected_collision_checking}
\end{align}
\end{definition}

\begin{definition}[Joint expected collision penalty]
For $n$ agents, the joint expected collision penalty is defined as:
\begin{align}
    J_c(p_R, p_1, \dots, p_n) & = \sum_{i=R}^{n} \sum_{j=i+1}^{n} c(p_i, p_j) \label{eq: J_c}
\end{align}
\end{definition}

\begin{definition}[Coupled Prediction and Planning in Distribution Space]
Given the original preference distributions for all agents $p_R(f), p_1(f), \dots, p_n(f)$, and the expected collision penalty function $c(p_i, p_j)$, the target is to find optimal preference distributions $(p_R, p_1, \dots, p_n)^*$ that minimize the joint expected collision penalty (\ref{eq: J_c}), and also prevent large deviations from the original preference distributions. Then the optimal joint \emph{trajectories} are selected individually as the maximum of each agent's optimal preference distribution.
\end{definition}

One point worth emphasizing in the above definition is that we do not put strict constraints on the difference between current preference distributions and original preferences, but instead consider it as a ``soft constraint'' while giving the joint expected collision penalty (\ref{eq: J_c}) a higher priority as the optimization objective. This is because the original preference distribution can lead to dangerous conflicts since they don't include other agents' intentions, in which case strictly constraining the deviation from such original preferences limits the agents' cooperation for safer joint actions. As will be discussed in the next section, in our proposed algorithm, we only constrain the preference deviation between two consecutive iterations during the optimization, rather than constraining them with respect to the original preferences;  this relaxation gives more freedom for conflict resolution.

We consider crowd navigation as a \emph{receding horizon} planning problem, therefore solving for optimal preferences distributions and selecting optimal coupled joint trajectories over one time step. When new observations are obtained, new preference distributions are generated and the joint trajectories for coupled planning and prediction are updated.

%% file: input/distribution-optimization.tex
\section{Variational Analysis for Coupled Prediction and Planning in Distribution Space} \label{sec: methodology}

\begin{figure}
    \centering
	\includegraphics[width=250pt]{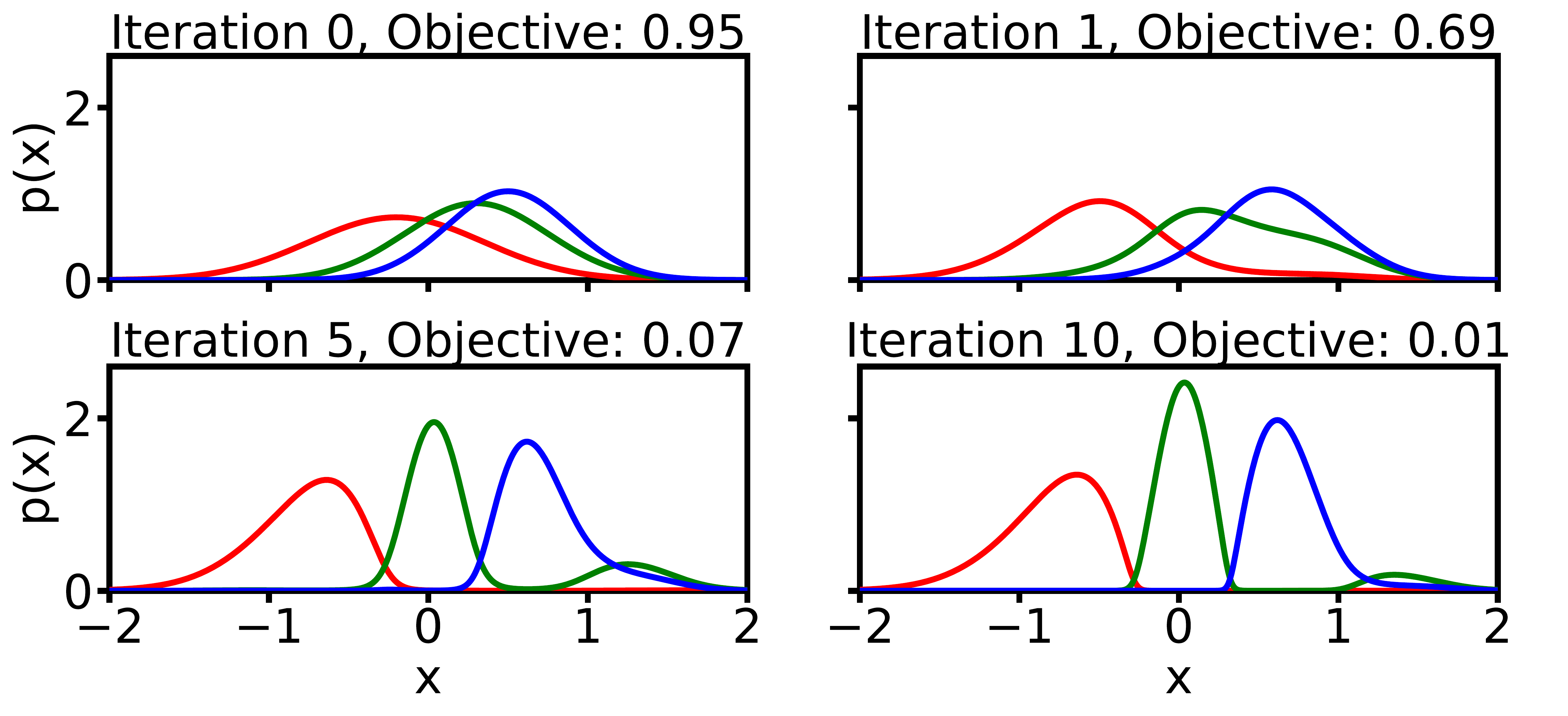}
    \caption{Evolution of three 1D preference distributions in 10 iterations using the sequential variational update. All three distribution are initialized as Gaussian distributions, the collision penalty function is a Gaussian probability density function with small variance. We can clearly see the change of intents and flexibilities (variance, skew and kurtosis) during the iterations.}
	\label{fig: three_agents_1D}
	\vspace{-15pt}
\end{figure}  

\begin{figure*}
    \centering
	\includegraphics[width=\textwidth]{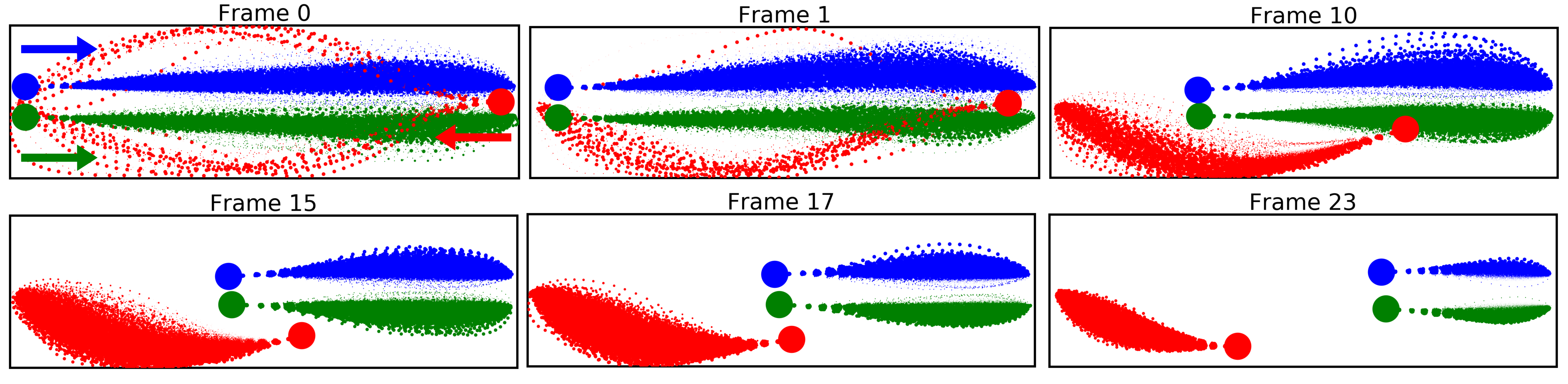}
    \caption{Evolution of three agents' preferences while they are passing each other in a narrow hallway, each preference distribution is approximated by 1000 samples, each sample's size represents its log-scaled weight. Each agent's intent is approximated by the sample trajectory with largest weight, the sample distribution reflects the flexibility information, such as the covariance (spread of the samples), skew (symmetry of the sample weights) and kurtosis (number of outlier samples). \textbf{DistNav captures the change of both \emph{intent} and \emph{flexibility} of each agent during the interaction. This can be clearly seen in the first two frames where the red agent is simultaneously reasoning about going to the left and to the right of the blue and green agents---something impossible to represent with trajectory space coupling.}}
	\label{fig: hallway_samples_1}
	\vspace{-10pt}
\end{figure*}  

The formulation of coupled prediction and planning in distribution space raises several challenges:
\begin{enumerate}
    \item The objective (\ref{eq: J_c}) is a functional, which is defined in the infinite-dimensional function space, so traditional optimization methods in vector space may not apply.
    \item The objective is subject to subsidiary constraints (\ref{eq: distribution_constraint}).
    \item The objective contains a combinatorial dependencies between the variables, in other words, all agents update their preferences in response to how others update. This structure makes the objective challenging to optimize. 
\end{enumerate}

To address the above challenges, in this paper we combine the following techniques:
\begin{enumerate}
    \item For the inter-agent dependencies, a strategy commonly used in multi-agent game solvers is the iterative best response (IBR) scheme \cite{wang2019game}\cite{spica2020real}. In one iteration, each agent's strategy (in our case the preference distribution) is locally optimized by solving a subproblem; in the subproblem other agents' strategies are fixed.
    \item The subproblem for each agent is constructed sequentially. Each agent updated their preference distribution in response to those who have already updated, and assuming those who haven't updated would follow. 
    \item Each subproblem can be solved analytically using Lagrange multipliers \cite{gelfand2000calculus} as an isoperimetric problem with subsidiary conditions.
\end{enumerate}

\subsection{Sequential Iterative Variational Update}

In $k$-th iteration, each agent updates its own preference sequentially by solving a subproblem. For agent $i$, the corresponding subproblem is:
\begin{align}
    p^{(k+1)}_i(f) & = \argmin_{p} \left\{ D_{KL}(p\Vert p^{(k)}_i) + \bar{c}_i^{(k)}(p) \right\} \label{eq: subproblem_objective}
\end{align} where
\begin{align}
    \bar{c}_i^{(k)}(p) & = \sum_{j=R}^{i-1} c(p, p_j^{(k+1)}) + \sum_{j=i+1}^{n} c(p, p_j^{(k)}) \label{eq: sequential_collision} \\
     & = \int_{\F} p(f) \bar{\gamma}_i^{(k)}(f) df \\
    \bar{\gamma}_i^{(k)}(f) & = \sum_{j=R}^{i-1} \int_{\F} \psi(f, f^{(j)}) p_j^{(k+1)}(f^{(j)}) df^{(j)} \\ 
     & + \sum_{j=i+1}^{n} \int_{\F} \psi(f, f^{(j)}) p_j^{(k)}(f^{(j)}) df^{(j)}
\end{align}

The first term in the subproblem objective (\ref{eq: subproblem_objective}) is the Kullback-Leibler divergence between the preferences in two iterations, which controls the change between preferences in two consecutive iterations, and therefore serves as a ``soft constraint'' on the deviation of current preference $p^{(k)}_i(f)$ from the original preference $p^{(0)}_i(f)$. The second term $\bar{c}_i^{(k)}(p)$ measures the summation of the expected collision penalties between agent $i$ and rest of the agents, assuming their preferences are fixed, under this assumption optimizing the second term equals optimizing the overall joint expected collision penalty (\ref{eq: J_c}). Note that the construction of $\bar{c}_i^{(k)}(p)$ needs to follow a sequential order as shown in (\ref{eq: sequential_collision}): all agents with indices smaller than $i$ would update their preferences for next iteration ahead of agent $i$, and their updated preferences are fixed in agent $i$'s subproblem. For all agents with indices larger than $i$, their preferences are fixed as before being updated for next iteration.

\begin{theorem} \label{theorem: subproblem_solution}
The global minimum for the subproblem (\ref{eq: subproblem_objective}) is:
\begin{align}
    p^{(k+1)}_i(f) = \frac{p^{(k)}_i(f)\exp(-\bar{\gamma}_i^{(k)}(f))}{\int_{\F} p^{(k)}_i(f)\exp(-\bar{\gamma}_i^{(k)}(f)) df} \label{eq: subproblem_solution}
\end{align}
\end{theorem}
\begin{proof}
See appendix.
\end{proof}

\begin{theorem} \label{theorem: sufficient_decrease}
After all $n$ agents' preferences have been updated sequentially in one iteration following (\ref{eq: subproblem_solution}), the inequality below holds, if $p_i^{(k+1)}(f)\neq p_i^{(k)}(f)$ for some $i\in\I$:
\begin{align}
    & J_c(p_R^{(k+1)}, p_1^{(k+1)}, \dots, p_n^{(k+1)}) \\
    & \leq J_c(p_R^{(k)}, p_1^{(k)}, \dots, p_n^{(k)}) - \xi \text{ , } \xi > 0
\end{align}
\end{theorem}
\begin{proof}
See appendix.
\end{proof}

Theorem \ref{theorem: sufficient_decrease} shows that in each iteration, the joint expected collision penalty (\ref{eq: J_c}) is guaranteed to be sufficiently decreased by updating the preference for each agent based on (\ref{eq: subproblem_solution}). Since (\ref{eq: J_c}) is bounded below at zero, with the decrease of its value, the subproblem (\ref{eq: subproblem_objective}) comes closer to optimizing the KL-divergence between the preference at two iterations; therefore the new preference at the next iteration comes closer to the one in last iteration. While we leave the question of whether the iterative update (\ref{eq: subproblem_solution}) could lead to guaranteed or efficient convergence to future work, in practice we don't actually look for the minimum of (\ref{eq: J_c}), which could be infeasible for navigation (e.g., a set of Dirac delta functions infinitely far from each other), but instead look for sufficiently small (\ref{eq: J_c}) that is safe enough. To terminate the iteration, one can threshold either the objective (\ref{eq: J_c}) or the KL-divergence between two iterations. Figure \ref{fig: three_agents_1D} shows an example of updating three 1D distributions using the update rule.

\subsection{DistNav: Sampling-Based Crowd Navigation Based On Sequential Iterative Variational Analysis}

Unfortunately, computing the preference updates (\ref{eq: subproblem_solution}) analytically in high dimensional space is intractable due to the integrals in (\ref{eq: subproblem_solution}), which is also known as the \emph{curse of dimensionality}. Therefore, we propose a sampling-based motion planner based on (\ref{eq: subproblem_solution}) to approximate the evolution of preferences and select the reference trajectory for robot navigation, the key idea here is to approximate the integrals through Monte Carlo integration.

We start by generating $m$ samples from the original preference distribution of each agent, here we denote the samples for agent $i$ as
\begin{align}
    [\f_{i}]^{(k)} & \sim p^{(k)}_i(f) \\
    [\f_{i}]^{(k)} & = \{\f^{(k)}_{i, 1}, \f^{(k)}_{i, 2}, \dots, \f^{(k)}_{i, m}\}
\end{align} where each sample $\f^{(k)}_{i, j}$ indicates the $j$-th sample of agent $i$ at $k$-th iteration, and it consists of two parts, the trajectory and the weight, we will not update the sample trajectory but only the weight:
\begin{align}
    \f_{i, j}^{(k)} & = (f_{i,j}, w_{i,j}^{(k)}), \quad f_{i,j} \in \F, \quad w_{i,j}^{(k)} \in \R{}
\end{align}
Before the iterations begin, the weight of each sample will be initialized as $w_{i,j}^{(0)}=1$. In the iteration, the weight of each sample will first be updated based on (\ref{eq: subproblem_solution}), where the integral of $\bar{\gamma}_i^{(k)}(f_{i,y})$ is approximated through Monte Carlo integration as:
\begin{align}
    \bar{\gamma}_i^{(k)}(f_{i,y}) & \approx \sum_{j=R}^{i-1} \left( \frac{1}{m} \sum_{z=1}^{m} \psi(f_{i,y}, f_{j,z}) \cdot w_{j,z}^{(k+1)} \right) \nonumber \\ 
     & + \sum_{j=i+1}^{n} \left( \frac{1}{m} \sum_{z=1}^{m} \psi(f_{i,y}, f_{j,z}) \cdot w_{j,z}^{(k)} \right) \label{eq: sample_weight_computation}
\end{align} And the weight of the sample is then updated as:
\begin{align}
    w_{i,j}^{(k+1)} = w_{i,j}^{(k)} \exp\left(-\bar{\gamma}_i^{(k)}(f_{i,y})\right)
\end{align} After all agent $i$'s sample weights have been updated, we normalize the weights such that the average weight remains $1$, this step is the approximation to the denominator in (\ref{eq: subproblem_solution}). For each sample $\f_{i, j}^{(k)} = (f_{i,j}, w_{i,j}^{(k)})$, the updated preference for the trajectory is approximated as:
\begin{align}
p_i^{(k)}(f_{i,j}) \approx w_{i,j}^{(k)} p_i^{(0)}(f_{i,j})
\end{align} After the iteration terminates, the optimal trajectory of each agent is selected as the sample with largest approximated preference. Pseudocode of the whole algorithm can be found in appendix. Figure \ref{fig: hallway_samples_1} shows how DistNav uses samples to approximate the evolution of three agents' preferences, where they pass each other in a narrow hallway.

%% file: input/evaluation.tex
\section{Evaluation} \label{sec: evaluation}

\subsection{Rationale for Both Simulated and Real World Dataset Evaluation}
For evaluation, we considered the crowd datasets ETH~\cite{walkalone} and UCY~\cite{ucy} and crowd simulators based on the the social forces model (SFM,~\cite{helbing1}; e.g., PEDSIM~\cite{pedsim}) and optimal reciprocal collision avoidance agents~\cite{rvo, rvoicra}, such as implemented in~\cite{alahi-crowd-attention}.  Ultimately, both evaluation methodologies have complementary strengths.  For example, simulated agents can be overly permissive with aggressive robots: our Monte Carlo IGP produced zero collisions in PEDSIM while quickly navigating to the goal, whereas in our ETH study it was unsafe in 34/181 of runs (row 9, Table~\ref{table:pt}) and exhibited freezing robot behavior. However, simulation provides information about how the algorithm leverages agent cooperation.  Alternatively, the humans in pre-recorded datasets are non-responsive, and so the robot cannot leverage cooperation. However, pre-recorded datasets provide \emph{human} benchmarks on safety and efficiency performance, which can be useful in assessing an algorithm's real world viability.   Thus, we used both validation techniques to gain insight about a) the algorithm's ability to leverage cooperation and b) the algorithm's safety and efficiency performance compared to human safety and efficiency performance.

However, \emph{how} to evaluate a navigation algorithm against ETH and UCY is non-trivial.  For instance, many trajectories in these datasets have no interaction; thus, the majority of the runs in these datasets are not sufficiently challenging for a navigation study.  Further, pre-recorded crowd datasets do not immediately suggest a \emph{navigation} testing protocol (ETH and UCY are typically used to benchmark \emph{prediction} algorithms, where testing protocol is straightforward).  However, a subsample of the ETH dataset (Figure~\ref{fig:eth-vis}; 100 frames, 150 pedestrians) collected for testing a deep network in~\cite{trajectron} has many interactions and substantial congestion; indeed, every pedestrian interacts at least once and most pedestrians interact many times during the 100 frame sequence.  To derive a \emph{navigation} test protocol, we expand on an idea from the experimental section of~\cite{trautmaniros}: 1) identify a pedestrian, 2) extract the start and end position of that pedestrian, 3) remove that pedestrian from the observation dataset of the navigation algorithm, and 4) provide the start and end positions of the removed pedestrian and the current and previous positions of the remaining agents to the navigation algorithm. Thus we assure that at least one path through the crowd exists (the one taken by the removed pedestrian). Additionally, by providing the navigation algorithm with start and end points that are joined by a path \emph{through} the crowd, the navigation algorithm naturally confronts high crowd densities (the human agent confronts an average density of 0.22 $people/m^2$ within a $3m$ radius circle).  Finally, this testing protocol provides us with a powerful performance benchmark: \emph{actual human performance on the exact same situation} as encountered by the navigation algorithm (first row, Table~\ref{table:pt}). To determine our safety threshold, we computed the shortest distance that any two humans in the ETH dataset ever came to each other; that distance was 0.21m.  Additionally, two humans only came within 0.3m of each other 3 times. Thus, if the robot is within 0.21m of any human, we consider that a \emph{collision}, while distances within 0.3m are considered \emph{unsafe} or \emph{uncomfortable}.

Furthermore, we partition this ETH dataset into what we call a ``partial'' trajectory dataset.  In the partial trajectory dataset, we consider all (approximately) 10 meter long agent runs.  For example, if agent 1's full trajectory was 30 meters long, we would have 3 partial trajectories.  Partial trajectory experiments provide focused examination of an algorithm's ability to navigate through congestion in a safe and efficient manner. We identified 293 partial trajectories and tested 181 of them (the rest are discarded for calibration reason). 

\begin{figure}[t!]
\vspace{-0pt} 
\centering 
\includegraphics[width=0.3\textwidth]{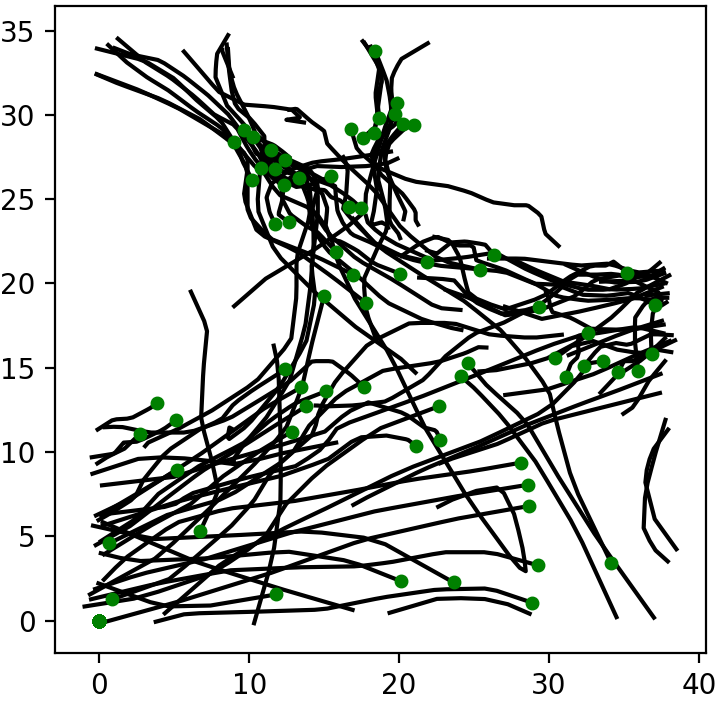}
\vspace{-0pt} 
\caption{First frame of the ETH data evaluated.  Pedestrian current position in green; next 40 time steps plotted as black curves.}
\label{fig:eth-vis} 
\vspace{-15pt}
\end{figure}

\begin{figure*}[t!]
\centering 
\includegraphics[width=\textwidth-20pt]{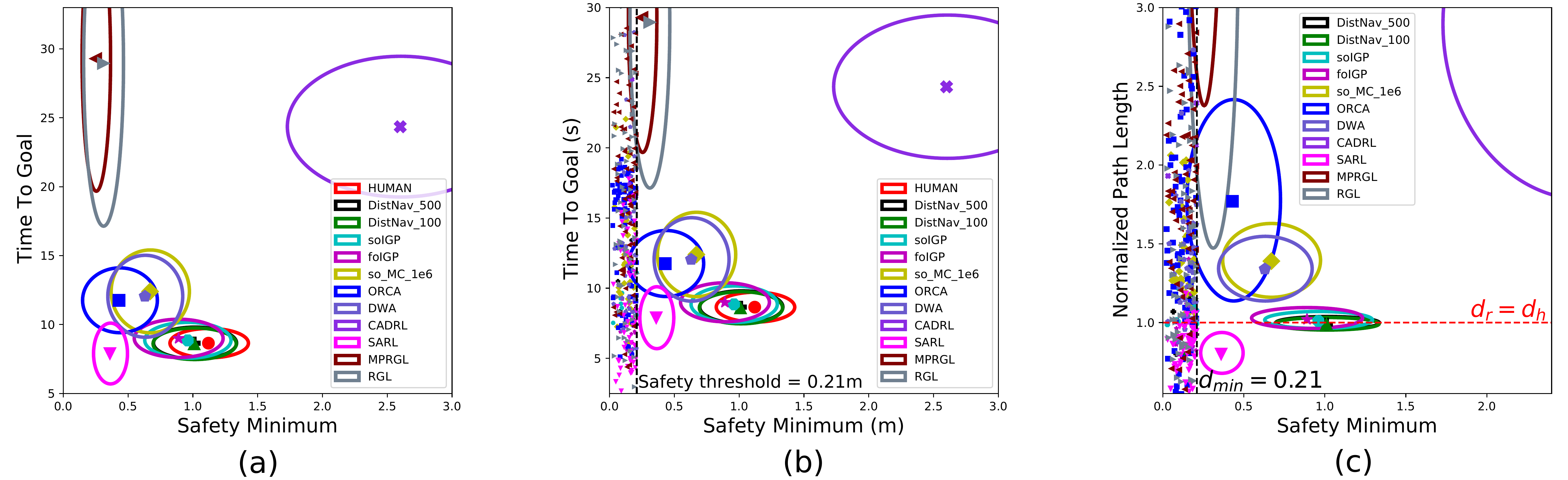}
\label{subfig:c}
\vspace{-0pt}
\caption{\textbf{Partial trajectory statistics}.  Total number of runs is 181; $x$ axes of a) and b) in meters.  All figures plot distance to nearest pedestrian  on $x$-axis; (a) plots means $\pm$ 1 standard deviation of algorithm and the human; (b) appends plot (a) with the $d_{min}$ threshold, the closest distance two humans passed in the dataset.  Inspection of the region left of $d_{min}$ shows numerous instances of DWA, so\_MC\_1e6, and SARL; (c) normalizes algorithm path length with human path length. E.g., values below $d_r=d_h$ mean that the robot moved to the goal more directly than the human.}
\label{fig:pt} 
\vspace{-0pt}
\end{figure*}
 
\subsection{Rationale for Test Algorithms}
We collected safety and path length data on humans, DistNav, ``first order'' interacting Gaussian processes (foIGP,~\cite{trautman-icaps-2020}), ``second order'' IGP (soIGP, the extension of foIGP that considers both robot-agent and agent-agent interactions), soIGP using Monte Carlo optimization with $10^6$ samples (so\_MC\_1e6, implementation details can be found in~\cite{trautman-icaps-2020}), the ``dynamic window approach'' (DWA,~\cite{dynamic-window}), and ORCA~\cite{rvo}.  For our deep reinforcement learning baselines, we tested against ``collision avoidance with deep reinforcement learning'' (CADRL,~\cite{crowd-nav-deep-learning-mit}),  ``socially aware reinforcement learning'' (SARL,~\cite{alahi-crowd-attention}), ``relational graph learning'' and ``model predictive relational graph learning'' (RGL, MP-RGL,~\cite{rgl-mprgl}).  

Each algorithm was chosen to explore a certain aspect of the performance space.  We collected data on humans to serve as an upper bound on performance.  We tested soIGP and foIGP as state of the art trajectory space approaches.  We tested DWA and ORCA because both are widely deployed; in particular, DWA is the default navigation algorithm in ROS (see ROS's \href{http://wiki.ros.org/base_local_planner}{base local planner }).  We tested against the 4 highest performing deep reinforcement learning variants to understand how model based and learning based algorithms fare against each other.  

Finally, we trained SARL in 7 different environments using the toolbox implemented in \cite{alahi-crowd-attention} \footnote{The toolbox is available at \url{https://github.com/vita-epfl/CrowdNav}}.  We trained in a 3m by 10m corridor (which mimics the ETH test conditions) with 15 and 5 people (0.5 and 0.16 people$/m^2$ densities), with mostly cross human cross traffic. The high density environment produced freezing robot behavior (freezing behavior $= 71\%$, $\max(d_r/d_h) = 23.8$, where $d_r$ and $d_h$ are the path lengths of the robot and human, respectively), while the low density training produced a policy that was unsafe (Collisions $= 34\%$).  We thus attempted training in the high density corridor, but with random start and goal positions of the people; this again resulted in freezing robot behavior at test  (freezing behavior $= 18\%$, $\max(d_r/d_h) = 14.7$).  We also trained in a 4m radius circular environment with 10 people, so the density was $\approx 0.2$ people$/m^2$ (the average density in the ETH data was $\approx 0.2$ people$/m^2$).  This policy also showed freezing robot behavior (freezing behavior $= 10\%$, $\max(d_r/d_h) = 4.32$).  Additionally, we trained the other DRL variants (CADRL, RGL, and MP-RGL) in both ORCA and SFM training environments, with 5, 7, 14, 21, and 28 people in a 4m radius circular environment.  Ultimately, the training regimen detailed in~\cite{alahi-crowd-attention}---a 4m radius circle with 5 agents---produced the best performing policy for all DRL variants.  We used these top performing policies for testing.

\subsection{Safety and Efficiency Evaluation on ETH}

\begin{table}[h!]
\centering
\begin{tabular}{l|llll}
\hline
                   		&Discomfort  &Collisions & \begin{tabular}[c]{@{}l@{}}Freezing\\ Behavior\end{tabular} &$\max(d_{r}/d_{h})$ \\ \hline
\textbf{Human}              	&\textbf{1.7\%}          & \textbf{0}             & \textbf{0\%} & \textbf{1} \\
\textbf{DistNav\_500}$^\ddag$\tnote{1}&\textbf{6.0\%}              & \textbf{3.0\%}      & \textbf{0\%} &  \textbf{1.18}  \\
\textbf{DistNav\_100}$^\ddag$\tnote{1}&\textbf{3.0\%}           & \textbf{1.0\%}      & \textbf{0\%}  & \textbf{1.18}       \\
soIGP              		&13.3\%        & 5\%         & 1\%  & 1.6  \\
foIGP              		&16.7\%        & 10.5\%    & 3\%  & 1.8 \\
so\_MC\_1e6$^\dagger$&30\%         & 18.8\%    & 51\%  & 5.3 \\
ORCA               		    &63\%       & 48.6\%    & 58\%  & 9.2   \\
DWA                		 &35\%          & 23.8\%    & 48\% & 4.1       \\
CADRL$^*$\tnote{2} 	 &12\%          & 6.6\%       & 80\%  & 14.7 \\
SARL$^*$\tnote{2}  	&50\%           & 31.5\%    & 0.5\% & 3.3 \\
RGL$^*$\tnote{2} 	&67.4\%        & 48\%       & 73\% & 28.1 \\
MP-RGL$^*$\tnote{2}	&62\%           & 39\%       & 88\% & 22.5 
\end{tabular}
\begin{tablenotes}
\item[1] $^\ddag$ DistNav\_500 and DistNav\_100 use 500 and 100 samples for each agent, respectively. The differences in the metrics between the two are partially from the sampling nature of the algorithm. A more comprehensive future evaluation would take multiple trials to eliminate such effects. 
\item[2] $^*$ SARL, CADRL, RGL, and MP-RGL were trained in 7 different environments; we report the best performing policy.
\end{tablenotes}
\caption{\textbf{ETH partial trajectory metrics}. ``Discomfort'' and ``Collisions'' are the percent of runs such that safety minimum distance $s<0.3m, 0.21m$; ``Freezing Behavior'' is the percent of robot path lengths 1.25 times longer than the corresponding human path length; $\max(d_{r}/d_{h})$ is the maximal ratio between the path lengths of the robot and human, it measures how \emph{inefficient} the algorithm is compared with human.}
\label{table:pt}
\vspace{-10pt}
\end{table}

\begin{figure*}
    \centering
    \includegraphics[width=\textwidth-30pt]{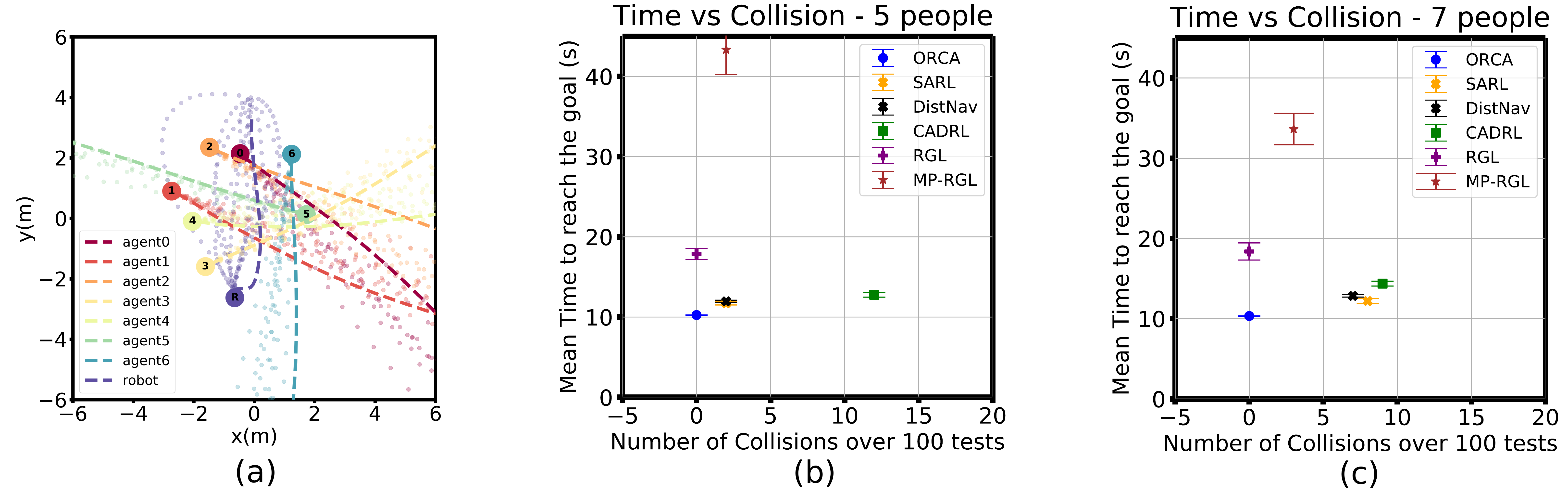}
    \caption{\textbf{ORCA evaluation statistics. }(a) shows one frame of the 7 agents test in ORCA simulation, the dashed lines are the robot's planned trajectory and optimal prediction for pedestrian trajectories, the colored dots are parts of the samples used to approximate preference distributions. (b) and (c) plot each algorithm's number of collisions vs. mean time to reach the goal over 100 tests with 5 and 7 pedestrians, respectively.}
    \label{fig: orca-sim}
    \vspace{-15pt}
\end{figure*}

The results of 181 partial trajectory runs are reported in Table \ref{table:pt} and Figure \ref{fig:pt} (a more comprehensive evaluation table can be found in the appendix). For safety and efficiency, only DistNav and soIGP are competitive with human performance, with DistNav outperforming soIGP. We tested DistNav with 100 samples per agent, in which case the algorithm can run in real time (average replanning time 0.23s). We also tested with 500 samples per agent, which resulted in longer computation time but no significant improvement in performance. In practice we think 100 samples per agent is a good balance between performance and computation efficiency.

Additionally, DWA and ORCA both exhibit freezing robot behavior (large value in ``freezing behavior'' column) and high rates of collision (large values in ``collision'' column). We note that all four DRL variants (last four rows of Table \ref{table:pt}) display extremely long paths (CADRL, RGL, MP-RGL), high rates of collision (SARL, RGL, MP-RGL) or both (RGL, MP-RGL).  Given the wide variety of testing protocols, these results indicate potential limitations of purely simulation-based training.

We emphasize that although ETH provides useful information, it does not  \emph{validate} that an algorithm will perform well in the real world (this requires real world experiments). Nevertheless, we believe that the ETH benchmark can provide evidence of \emph{invalidation}; that is, poor performance on ETH possibly indicates that an algorithm is not suitable for real world deployment.

\subsection{Safety and Efficiency Evaluation in Simulation}

We complement our ETH study with two simulation studies: one with 5 ORCA agents (the standard testing environment for DRL studies~\cite{alahi-crowd-attention, crowd-nav-deep-learning-mit, rgl-mprgl}) and one with 7 agents.  Our ORCA simulators are closely related to those developed in~\cite{alahi-crowd-attention}, with a few important differences:
\begin{itemize}
    \item The robot is visible in both the 5 and 7 agent test because we wish to understand how our algorithm performs in the presence of responsive agents.
    \item Once an ORCA agent reaches a goal, a new goal is provided to the agent.  In this way, the agents circulate in the work space and the crowd density remains consistent throughout the run. The simulator in~\cite{alahi-crowd-attention} provides a single goal to each agent; the agent stops once it arrives.  Thus, our simulator has a higher average crowd density than the simulator in~\cite{alahi-crowd-attention}.
\end{itemize}

We point out that the ORCA \emph{robot} outperforms all the algorithms in both the 5 and 7 person simulation; this is to be expected, since a group of ORCA agents have guarantees on collision performance and locally optimal efficiency.  In short, an ORCA robot is perfectly tuned to an ORCA simulation.  However, as seen in Table~\ref{table:pt}, ORCA is unsuitable for deployment in scenarios with agents not obeying the ORCA protocol.

For the five person simulation, we see that only ORCA and RGL outperform DistNav in terms of safety; however, RGL exhibits substantially longer mean time to goal, indicating freezing robot like behavior (e.g., RGL often chooses to go around the crowd).  SARL shows nearly identical performance to DistNav, but its critical to recall that SARL (and the other DRL variants) were specifically trained in a 5 person ORCA simulator.  Thus, all the DRL variants have a large advantage over DistNav: they have been precisely tuned to this simulation, whereas DistNav has not.  In combination with the ETH results in Figure~\ref{fig:pt} and Table~\ref{table:pt}, DistNav displays substantially stronger performance, both in terms of safety and efficiency, and, more importantly, in the ability to generalize to novel scenarios.  The 7 person case shows nearly identical qualitative results (although the exact number of collisions or time to goal changes slightly, the ordering of the algorithms remains the same.)

Finally, we attempted tests with the number of ORCA agents higher than 7, but this led to hard-to-interpret results because of simulator failures.  For example, at higher densities, agents can be so close together that collisions are often \emph{caused} by the ORCA agents themselves (e.g., an ORCA agent runs into the robot).  While testing in simulation at higher densities is important, fixing the simulator is out of scope of this paper.

%% file: input/conclusion.tex
\section{Conclusion} \label{sec:conclusion}

We studied the crowd navigation problem by modeling both human and robotic actions as probability density functions (called \emph{preference distributions}). This formulation, together with an optimization algorithm, captures the evolution of agents' preferences in the presence of interaction, something not modeled using trajectory space coupling.  Further, we designed a sampling-based crowd navigation method, called DistNav, and benchmarked against a variety of methods in both a real world dataset and in simulation. In both the dataset and simulation evaluation, Distnav outperformed all other algorithms and was competitive with human safety and efficiency performance. 

%% file: input/appendix.tex
\begin{appendix} 
\section{sec: appendix}

\subsection{Evaluation on ETH dataset with Supplementary Metrics}
See Table \ref{table:pt_full}.


\begin{table*}[h!]
\centering
\begin{tabular*}{\textwidth}{l | @{\extracolsep{\fill}} lllllll}
\hline
                   		&Discomfort  &Collisions & \begin{tabular}[c]{@{}l@{}}Freezing\\ Behavior\end{tabular}   &$\max(d_{r}/d_{h})$     & $\mu(s)$           & $\mu(d_{r})$                          & $\mu(t)$ \\ \hline
Human              	&1.7\%          & 0             & 0\%             & 1		& $1.1 \pm .2m$  &$ 8.7\pm 1.0m$                      & NA  \\
DistNav\_500$^\ddag$\tnote{1}&6.0\%              & 3.0\%      & 0\%           & 1.18		& $1.01 \pm .32m$     &$ 8.65\pm1.15m$                                                & $1.07\pm0.83s$ \\
DistNav\_100$^\ddag$\tnote{1}&3.0\%           & 1.0\%      & 0\%           & 1.18		& $1.01 \pm .32m$  &$ 8.6\pm1.14m$                                                 & $0.26\pm0.07s$ \\
soIGP              		&13.3\%        & 5\%         & 1\%           & 1.6 		& $.96 \pm .3m$   &$ 8.9\pm1.3m$                                                  & $4.3\pm.2s$ \\
foIGP              		&16.7\%        & 10.5\%    & 3\%           & 1.8		& $ .9 \pm .3m$    &$ 9\pm 1.5m$                                                     &  $4.2\pm.15s$ \\
so\_MC\_1e6$^\dagger$\tnote{2}&30\%         & 18.8\%    & 51\%        & 5.3		& $ .67\pm .3m$   &$ 12.4\pm 3m$                                                    & $6.1\pm0.8s$$^\dagger$\tnote{2} \\
ORCA               		    &63\%       & 48.6\%    & 58\%        & 9.2		& $.43 \pm .3m$    &$11.8\pm 2.4m$                                                 & $0.03\pm.001s$  \\
DWA                		 &35\%          & 23.8\%    & 48\%        & 4.1		 & $ .6\pm .4m$      &$ 12.1\pm 2.5m$                                               & $.1\pm.03s$ \\
CADRL$^*$\tnote{3} 	 &12\%          & 6.6\%       & 80\%       & 14.7		& $2.6 \pm .9m$    &$ 24.35\pm 5.1m$                                        & $2.9\pm.4s$  \\
SARL$^*$\tnote{3}  	&50\%           & 31.5\%    & 0.5\%          & 3.3		& $.4 \pm .15m$    &$ 7.9\pm 2.2m$                                                   & $4.4\pm1.1s$  \\
RGL$^*$\tnote{3} 	&67.4\%        & 48\%       & 73\%       &28.1		&$0.3\pm.15m$     &$ 28.95\pm 12.6m$                                       & $2.1\pm.7s$ \\
MP-RGL$^*$\tnote{3}	&62\%           & 39\%       & 88\%      &22.5		 & $0.33\pm.16m$  &$ 38.1\pm 12.5m$                                         & $0.3\pm.01s$  
\end{tabular*}
\begin{tablenotes}
\item[1] $^\ddag$ DistNav\_500 and DistNav\_100 use 500 and 100 samples for each agent, respectively. They were both run on a 12 core CPU, parallelized and accelerated by the Numba Python package~\cite{lam2015numba}, but no GPU is used.
\item[2] $^\dagger$ so\_MC\_1e6 was run on fully parallelized code on a 64 core CPU to achieve these times.
\item[3] $^*$ SARL, CADRL, RGL, and MP-RGL were trained in 7 different environments; we report the best performing policy.
\end{tablenotes}
\caption{\textbf{ETH supplementary partial trajectory metrics}.  Distance to nearest pedestrian is $s$ and $\mu(s)$ is the mean; ``Discomfort'' and ``Collisions'' are the percent of runs such that $s<0.3m, 0.21m$; ``freezing behavior'' is the percent of robot path lengths 1.25 times longer than the corresponding human path length; $\max(d_{r}/d_{h})$ is the largest value of robot path length divided by corresponding human path length; $\mu(d_r)$ is mean robot path length $d_r$ over all runs;  $\mu(t)$ is mean time of all replanning steps.}
\label{table:pt_full}
\vspace{-10pt}
\end{table*}

\subsection{Pseudocode of DistNav}

See Algorithm \ref{alg: sampling_method}.

\begin{algorithm} [ht] 
    \caption{Sampling-Based Crowd Navigation Based On Sequential Iterative Variational Analysis (DistNav)} \label{alg: sampling_method}
    \SetKwInOut{KwIn}{Input}
    \SetKwInOut{KwOut}{Output}
    \KwIn{$[[\f_0], [\f_1], \dots, [\f_n]]$: Samples representing initial preferences of $n+1$ agents, each agent has $m$ samples. The weight of each sample is initialized as 1. Index $0$ indicates the robot. \\
          $\psi(x_1,x_2)$: collision penalty function. \\
          $[p_1(f), p_2(f), \dots, p_n(f)]$: Original preference\\ distributions of $n+1$ agents. \\
          $\epsilon$: Termination condition.}
    \KwOut{$[f_0^*, f_1^*, \dots, f_n^*]$: Optimal trajectories of $n$ agents selected from samples. Index $0$ indicates the robot.\\}
    $i \leftarrow 0$ \\
    \While{$i\leq n$}{
        $[\f_i]^{(0)} \leftarrow [\f_i]$ \\
        $i \leftarrow i+1$
    }
    $k \leftarrow 0$ \\
    $objective \leftarrow \frac{1}{m}\sum_{i=0}^{n}\sum_{j=i+1}^{n}\sum_{l=0}^{m} \left( \psi(f_{i,m}, f_{j,m}) \cdot w_{i,m}^{(k)} \cdot w_{j,m}^{(k)} \right)$
    \While{$objective \geq \epsilon$}{
        $i \leftarrow 0$ \\
        \While{$i\leq n$}{
            $j \leftarrow 1$ \\
            \While{$j\leq m$}{
                $v \leftarrow \sum_{l=0}^{l=i-1} \sum_{h=1}^{h=m} \left( \psi(f_{i,j}, f_{l,h}) \cdot w_{l,h}^{(k+1)} \right) + \sum_{l=i+1}^{l=n} \sum_{h=1}^{h=m} \left( \psi(f_{i,j}, f_{l,h}) \cdot w_{l,h}^{(k)} \right)$ \\
                $w_{i,j}^{(k+1)} \leftarrow w_{i,j}^{(k)} \cdot \exp(-\frac{v}{m})$ \\
                $j \leftarrow j + 1$
            }
            $j \leftarrow 1$ \\
            \While{$j\leq m$}{
                $w_{i,j}^{(k+1)} \leftarrow w_{i,j}^{(k+1)} / \left(\frac{1}{m} \sum_{l=0}^{m} w_{i,j}^{(k+1)}\right)$ \\
                $j \leftarrow j+1$
            }
            $i \leftarrow i + 1$
        }
        $objective \leftarrow \frac{1}{m}\sum_{i=0}^{n}\sum_{j=i+1}^{n}\sum_{l=0}^{m} \left( \psi(f_{i,m}, f_{j,m}) \cdot w_{i,m}^{(k)} \cdot w_{j,m}^{(k)} \right)$\\
        $k \leftarrow k + 1$
    }
    $i \leftarrow 0$ \\
    \While{$i\leq n$}{
        $f_i^* \leftarrow \argmax_{f_{i,j}^{(k)}} p_i(f_{i,j}^{(k)}) w_{i,j}^{(k)}$ \\
        $i \leftarrow i+1$
    }
    \Return $[f_0^*, f_1^*, \dots, f_n^*]$
\end{algorithm}

\subsection{Proof}

\noindent\textbf{Proof for Theorem \ref{theorem: subproblem_solution}}
\begin{proof}
The subproblem can be considered as an isoperimetric problem with a subsidiary condition (\ref{eq: distribution_constraint}), therefore we first formulate the Lagrangian as
\begin{align}
    \mL(p, \lambda) & = D_{KL}(p\Vert p^{(k)}_i) + \bar{c}_i^{(k)}(p) - \lambda (\int_{\F} p(f) df - 1) \\
    \bar{c}_i^{(k)}(p) & = \int_{\F} p(f) \bar{\gamma}_i^{(k)}(f) df
\end{align} where $\lambda\in\R{}$ is Lagrange multiplier. The necessary condition for $p^*(f)$ to be an extremum for the subproblem can be written as (Theorem 1, Page 43 \cite{gelfand2000calculus}):
\begin{align}
    \frac{\partial \mL}{\partial p}(p^*,\lambda) & = \log p^*(f) + 1 - \log p_i^{(k)}(f) + \bar{\gamma}_i^{(k)}(f) - \lambda = 0 \\
    p^*(f) & = p_i^{(k)}(f) \exp(-\bar{\gamma}_i^{(k)}(f)+\lambda-1) \label{eq: equality_plugin_1}
\end{align}
By substituting (\ref{eq: equality_plugin_1}) into the equality constraint (\ref{eq: distribution_constraint}), we can solve for $\lambda$:
\begin{align}
    \int_{\F} p^*(f) df & = \int_{\F} p_i^{(k)}(f) \exp(-\bar{\gamma}_i^{(k)}(f)+\lambda-1) df \\
    & = \exp(\lambda-1) \int_{\F} p_i^{(k)}(f) \exp(-\bar{\gamma}_i^{(k)}(f)) df = 1 \\
    \exp(\lambda-1) & = \frac{1}{\int_{\F} p_i^{(k)}(f) \exp(-\bar{\gamma}_i^{(k)}(f)) df} \label{eq: equality_plugin_2}
\end{align}
Substituting (\ref{eq: equality_plugin_2}) into (\ref{eq: equality_plugin_1}) gives us:
\begin{align}
    p^*(f) = \frac{p^{(k)}_i(f)\exp(-\bar{\gamma}_i^{(k)}(f))}{\int_{\F} p^{(k)}_i(f)\exp(-\bar{\gamma}_i^{(k)}(f)) df}
\end{align} Since the subproblem objective is unbounded from above, the solution $p^*(f)$ is a global minimum, which completes the proof.
\end{proof}

\begin{lemma} \label{lemma: subproblem_decrease}
The subproblem solution (\ref{eq: subproblem_solution}) can sufficiently decrease the second term $\bar{c}_i^{(k)}(p)$ in the subproblem objective (\ref{eq: subproblem_objective}), if $p_i^{(k+1)}(f) \neq p_i^{(k)}(f)$:
\begin{align}
    \bar{c}_i^{(k)}(p_i^{(k+1)}) & \leq \bar{c}_i^{(k)}(p_i^{(k)}) - \xi \\
    \xi & > 0
\end{align}
\end{lemma}
\begin{proof}
Since $p^{k+1}_i$ is the global minimum of the subproblem (\ref{eq: subproblem_objective}), we have:
\begin{align}
    & D_{KL}(p^{(k+1)}_i\Vert p^{(k)}_i) + \bar{c}_i^{(k)}(p_i^{(k+1)}) \\
    & \leq D_{KL}(p^{(k)}_i\Vert p^{(k)}_i) + \bar{c}_i^{(k)}(p_i^{(k)}) = \bar{c}_i^{(k)}(p_i^{(k)}) \\
    & \bar{c}_i^{(k)}(p_i^{(k+1)}) \leq \bar{c}_i^{(k)}(p_i^{(k)}) -  D_{KL}(p^{(k+1)}_i\Vert p^{(k)}_i)
\end{align}
If $p_i^{(k+1)}(f) \neq p_i^{(k)}(f)$, then $D_{KL}(p^{(k+1)}_i\Vert p^{(k)}_i) > 0$, which completes the proof.
\end{proof}

\noindent\textbf{Proof for Theorem \ref{theorem: sufficient_decrease}}
\begin{proof}
Based on Lemma \ref{lemma: subproblem_decrease}, we have for each $i\in\I$, the following inequality holds:
\begin{align}
    \bar{c}_i^{(k)}(p_i^{(k+1)}) \leq \bar{c}_i^{(k)}(p_i^{(k)}) -  D_{KL}(p^{(k+1)}_i\Vert p^{(k)}_i)
\end{align} Summing up left hand side of the inequality for all $i\in\I$ gives us:
\begin{align}
    & \sum_{i=R}^{n} \bar{c}_i^{(k)}(p_i^{(k+1)}) \nonumber \\
    & = \sum_{i=R}^{n} \sum_{j=R}^{i-1} c(p_i^{(k+1)}, p_j^{(k+1)}) + \sum_{i=R}^{n} \sum_{j=i+1}^{n} c(p_i^{(k+1)}, p_j^{(k)}) \\
    & = \sum_{i=R}^{n} \sum_{j=i+1}^{n} c(p_i^{(k+1)}, p_j^{(k+1)}) + \sum_{i=R}^{n} \sum_{j=R}^{i-1} c(p_i^{(k)}, p_j^{(k+1)})
\end{align} The last equality above is based on the structure of the combinatorial summation and the fact that $c(p_i,p_j)=c(p_j,p_i)$.
Meanwhile summing up the right hand side of the inequality for all $i\in\I$ gives us:
\begin{align}
    & \sum_{i=R}^{n} \bar{c}_i^{(k)}(p_i^{(k)}) - \sum_{i=R}^{n} D_{KL}(p^{(k+1)}_i\Vert p^{(k)}_i) \\ 
    & = \sum_{i=R}^{n} \bar{c}_i^{(k)}(p_i^{(k)}) - \xi \\
    & = \sum_{i=R}^{n} \sum_{j=R}^{i-1} c(p_i^{(k)}, p_j^{(k+1)}) + \sum_{i=R}^{n} \sum_{j=i+1}^{n} c(p_i^{(k)}, p_j^{(k)}) - \xi 
\end{align} Now by the combining summation of both sides of the inequality, we would have:
\begin{align}
    & \sum_{i=R}^{n} \sum_{j=i+1}^{n} c(p_i^{(k+1)}, p_j^{(k+1)}) + \sum_{i=R}^{n} \sum_{j=R}^{i-1} c(p_i^{(k)}, p_j^{(k+1)}) \\
    & \leq \sum_{i=R}^{n} \sum_{j=R}^{i-1} c(p_i^{(k)}, p_j^{(k+1)}) + \sum_{i=R}^{n} \sum_{j=i+1}^{n} c(p_i^{(k)}, p_j^{(k)}) - \xi 
\end{align} and therefore
\begin{align}
    & \sum_{i=R}^{n} \sum_{j=i+1}^{n} c(p_i^{(k+1)}, p_j^{(k+1)}) \leq \sum_{i=R}^{n} \sum_{j=i+1}^{n} c(p_i^{(k)}, p_j^{(k)}) - \xi 
\end{align}
If $p_i^{(k+1)}(f)\neq p_i^{(k)}(f)$ for some $i\in\I$, then $\xi=\sum_{i=R}^{n} D_{KL}(p^{(k+1)}_i\Vert p^{(k)}_i)>0$, which completes the proof.
\end{proof}

\subsection{Implementation Details for DistNav}

\subsubsection{Choice of Collision Penalty Function}
The collision penalty for two trajectories should be evaluated on time-aligned elements (poses at each time step), in our implementation we use a Gaussian penalty function and select the maximal collision penalty among all time steps. Even though theoretically a Dirac delta function should work as collision penalty since preference distribution already contains information about ``comfort distance'', for the sampling-based method an explicit collision penalty is still necessary.

\subsubsection{Selection of Critical Agents}

After generating the initial samples for all agents, we compute the weights of all other agents' GP preferences on the robot's intent (GP mean) based on (\ref{eq: sample_weight_computation}) as the ``interaction score''. Agents with scores higher than a user-defined threshold are considered as critical agents for interaction, and thus are included for coupled prediction and planning. This process could drastically reduce the computation time.

\end{appendix}